\documentclass[runningheads]{llncs}

\usepackage{eccv}

\usepackage{eccvabbrv}

\usepackage{graphicx}
\usepackage{booktabs}

\usepackage[accsupp]{axessibility}  

\usepackage{multirow}
\usepackage{algorithm}
\usepackage{algorithmic}
\usepackage{wrapfig}
\usepackage{tabularx}
\usepackage{amsmath}
\DeclareMathOperator*{\argmin}{arg\,min}
\usepackage{enumitem}

\usepackage{booktabs}
\usepackage{multirow}
\usepackage{colortbl}
\usepackage{xcolor}
\usepackage{makecell}
\usepackage{marvosym}

\definecolor{lightgray}{gray}{0.93}
\definecolor{bestcolor}{HTML}{E8F5E9}   

\usepackage{pifont}
\newcommand{\xmark}{{\color{red!70!black}\ding{55}}}

\usepackage{booktabs}
\usepackage{multirow}
\usepackage{colortbl}
\usepackage{xcolor}
\usepackage{makecell}
\usepackage{amssymb}   

\definecolor{lightgray}{gray}{0.93}
\definecolor{teal}{HTML}{00796B}

\newcommand{\sss}[1]{{\scriptsize\textcolor{teal}{\,(#1)}}}

\newcommand{\bd}[1]{\textbf{#1}}

\newcommand{\gn}[1]{{\small\textcolor{teal}{#1}}}

\newcommand{\romark}[1]{\makecell{#1}}

\usepackage{hyperref}

\usepackage{orcidlink}

\begin{document}

\title{MIRROR: Aligning Semantic Relations from Language to Image via Gromov--Wasserstein} 

\titlerunning{MIRROR}



\author{
Hong-Han Wang\thanks{Equal contribution.}\orcidlink{0009-0009-6976-3691}
\and
Yuntao Wang\textsuperscript{*}\orcidlink{0000-0002-1041-0252}
\and
Hu Ding\thanks{Corresponding author.}\orcidlink{0000-0002-1307-6077}
}

\authorrunning{H.-H. Wang et al.}

\institute{
University of Science and Technology of China, Hefei, China\\
\email{hh9999@mail.ustc.edu.cn, wangyuntao@mail.ustc.edu.cn, huding@ustc.edu.cn}\\
}

\maketitle

\begin{abstract}

Multimodal Large Language Models (MLLMs) inherit rich relational priors from their language backbones, yet often fail when asked to apply these relationships in visual contexts. We trace this failure to a structural blind spot: projection-based alignment trains each visual token to carry the right semantics, but never asks whether the relationships between concepts survive the crossing from language to vision. To address this, we propose \textbf{MIRROR} (\textbf{M}apping \textbf{I}nter-concept \textbf{R}elations from language to visual \textbf{R}epresentation via \textbf{O}ptimal-transport-based \textbf{R}egularization), a geometric regularization framework that transfers relational priors from language to vision by exploiting the rich relational structure encoded in language representations. Specifically, we derive a surrogate loss from the proposed \textbf{Semi-Inverse Gromov--Wasserstein (SI-GW)} problem, an inverse geometric problem that aligns visual representations with language-derived relational priors. We show that this formulation admits a unique closed-form solution that prescribes the ideal visual relational structure implied by language geometry and cross-modal coupling. The structure of the formulation also enables efficient computation, making it applicable to long token sequences. Applying SI-GW inside decoder-only Transformers requires careful design. We introduce targeted strategies at the layer, head, and token levels to ensure stable extraction without additional parameters or inference cost. MIRROR improves relational consistency while preserving performance on general vision-language tasks.

\keywords{Alignment \and Multimodal Large Language Model \and Vision-Language Model}

\end{abstract}

\section{Introduction}
\label{sec:intro}

Multimodal Large Language Models (MLLMs) have recently achieved substantial progress in visual understanding. The models such as GPT-4V~\cite{achiam2023gpt}, LLaVA~\cite{liu2024visual}, and Qwen-VL~\cite{bai2025qwen2} can perform complex visual question answering, comprehend scene relationships, and engage in multi-turn visual reasoning. These abilities arise largely from their powerful language model backbones, which encode extensive world knowledge and structured relational reasoning.

Yet, when MLLMs must apply such relational knowledge visually, a challenge emerges (\cref{fig:motivation}). A Large Language Model (LLM) that can textually reason that ``sitting requires the center of gravity above the support surface'' or ``Huskies differ from Malamutes in ear shape and body proportion'' may still fail to determine spatial support relations or confuse fine-grained categories in vision. This observation raises a fundamental question:
{Why do \textbf{semantic relationships}, that is, relational priors  already encoded in language models such as spatial configuration and logical dependency, fail to transfer to visual relation understanding?}


Specifically, mainstream MLLMs establish cross-modal connections 
through projection adapters~\cite{liu2024visual,li2023blip,alayrac2022flamingo} 
optimized under standard language modeling objectives. 
While this ensures that each visual token carries sufficient 
semantic information, it enforces alignment at the level of 
identity rather than relational 
structure~\cite{li2025vistaenhancingvisiontextalignment,masry2025alignvlmbridgingvisionlanguage}, 
as we formally analyze in \cref{sec:method_formulation}.

\begin{wrapfigure}{r}{0.48\linewidth}
  \centering
  \vspace{-20pt}
  \includegraphics[width=\linewidth]{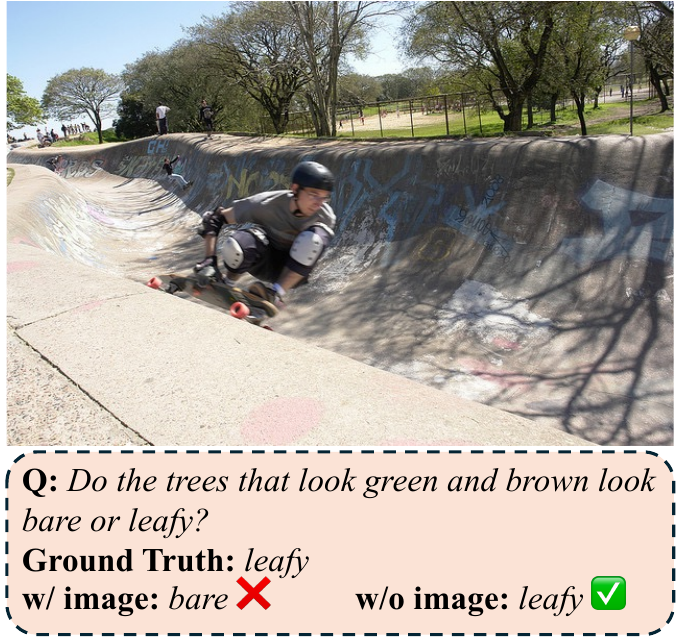}
  \vspace{-15pt}
  \caption{\textbf{Vision can override correct language priors.} A GQA~\cite{hudson2019gqa} example where Qwen2-VL-7B predicts ``leafy'' from ``green and brown'' without the image, but salient bare branches bias it toward an incorrect answer. See Appendix A for more examples.}
  \label{fig:motivation}
  \vspace{-20pt}
\end{wrapfigure}

This observation led us to a key insight: effective cross-modal understanding requires aligning not just what concepts mean, but how they relate to each other, that is, the geometric structure encoding semantic relationships. Recent theoretical works support this perspective. Huh~\cite{huh2024position} proposed the Platonic Representation Hypothesis, suggesting that representations across modalities converge toward shared statistical structures, while other research~\cite{han2025learningseeingdemystifyingllm} shows that large language models spontaneously develop visual reasoning priors from structured text, particularly relational structures between concepts. Together, these findings suggest that text, due to its symbolic and compositional nature, encodes semantic relations in a more structured geometry, making it a valuable ``teacher'' for improving visual relational reasoning.


To model the transfer of relational structure across modalities, we consider 
Optimal Transport (OT)~\cite{villani2009optimal}, a principled framework for comparing probability 
distributions while preserving the geometry of their underlying spaces.  
An important extension is the Gromov--Wasserstein (GW) distance~\cite{memoli2011gromov}, 
which aligns two metric measure spaces 
$(\mathcal{X}, d_{\mathcal{X}}, \mu)$ and $(\mathcal{Y}, d_{\mathcal{Y}}, \nu)$ 
by minimizing the distortion between their pairwise distance structures under 
a joint coupling.  
For discrete distributions $\mu$ and $\nu$ supported on 
$n_t$ text and $n_v$ visual tokens respectively, the admissible coupling set is

\[
\Pi(\mu,\nu)=
\{C \in \mathbb{R}_+^{n_t \times n_v} \mid 
C \mathbf{1}_{n_v} = \mu,\; 
C^\top \mathbf{1}_{n_t} = \nu\},
\]
where $C$ is a joint probability matrix whose row and column marginals match $\mu$ and $\nu$.
The squared GW distance is then defined as
\begin{equation}
\mathrm{GW}^2(\mu,\nu)
=
\min_{C \in \Pi(\mu,\nu)}
\sum_{i,j=1}^{n_t}
\sum_{k,\ell=1}^{n_v}
\bigl| d_\mathcal{X}(i,j) - d_\mathcal{Y}(k,\ell) \bigr|^2 
\, C_{ik} C_{j\ell}.
\label{eq:gw_distance}
\end{equation}
We denote $D^X = (d_\mathcal{X}(i,j))_{n_t \times n_t}$ and 
$D^Y = (d_\mathcal{Y}(k,\ell))_{n_v\times n_v}$ as the pairwise distance matrices
on $\mathcal{X}$ and $\mathcal{Y}$.

While GW provides a principled measure of structural discrepancy, the problem we face in MLLMs is fundamentally different from the one GW solves.
Standard GW takes two fixed metric spaces and searches for an optimal coupling between them.
In our setting, however, the cross-attention mechanism of the MLLM already provides a natural, data-dependent coupling between text and visual tokens at no extra cost. What is missing is not the correspondence, but rather a precise specification of what relational structure the visual space should exhibit given the well-organized language geometry.
This reframes cross-modal relational alignment as an \textbf{inverse geometric problem}:

\textit{Given the text geometry $D^t$ and the coupling $C$ supplied by cross-attention, what is the ideal visual geometry $\widehat{D}^v$ that would make the two spaces relationally consistent under $C$?}

We formalize this question as the \textbf{Semi-Inverse Gromov--Wasserstein (SI-GW)} problem, which optimizes over the visual distance matrix $D^v$ within the GW objective while anchoring both $D^t$ and $C$.
We prove that this inverse problem admits a unique closed-form solution,  which maps text-space relational structure onto the visual token space through the cross-attention coupling.
Moreover, directly optimizing a visual geometry to preserve pairwise relational consistency requires jointly reasoning over all pairs of text tokens and all pairs of visual tokens, incurring a quartic cost that is prohibitive for the long token sequences typical in MLLMs.
We show that the structure of the inverse problem can be carefully exploited, so that its solution can be decomposed into a sequence of basic matrix operations, reducing the cost of computing the objective of SI-GW loss from $O(n_t^2 n_v^2)$ to $O(n_t^2 n_v + n_v^2 n_t)$, where $n_t$ and $n_v$ denote the numbers of text and vision tokens, respectively.

Deploying SI-GW within large-scale MLLMs is not straightforward, however, as the fidelity of its output depends critically on the quality of both the anchored text geometry and the coupling between modalities. Naively sourcing these from arbitrary Transformer layers and heads may introduce noise and gradient instability. We therefore design a set of principled implementation strategies, including layer decoupling, selective head aggregation, and token suppression, that ensure each component of SI-GW receives clean, semantically meaningful input from the Transformer architecture.

The resulting framework, \textbf{MIRROR}, \textbf{M}apping \textbf{I}nter-concept \textbf{R}elations from language to visual \textbf{R}epresentation via \textbf{O}ptimal-transport-based \textbf{R}egularization,
introduces no additional parameters or inference cost. Moreover, the experiments across relational and general vision-language
benchmarks show that MIRROR 
consistently improves inter-concept reasoning while 
maintaining promising performance on general vision-language tasks.

\vspace{-10pt}
\section{Related Work}
\label{sec:related}
\vspace{-5pt}
\paragraph{Multimodal Large Language Models and Cross-Modal Alignment.}
Recent multimodal large language models 
(MLLMs)~\cite{zhu2023minigpt,dai2024instructblip,liu2024visual,bai2025qwen2,lin2024vila,yao2024minicpm} 
integrate large language 
models~\cite{touvron2023llama,jiang2023mistral,chiang2023vicuna} 
with vision encoders~\cite{radford2021learning,li2022grounded} 
through learnable projection 
modules~\cite{liu2024visual,li2023blip, chen2024sharegpt4v}. 
While this adapter-based paradigm effectively maps visual 
tokens into the language embedding space, the alignment is 
typically driven by next-token prediction alone, which 
supervises each visual token independently without 
explicitly preserving inter-concept relational 
structure~\cite{li2025vistaenhancingvisiontextalignment, masry2025alignvlmbridgingvisionlanguage}. 
Meanwhile, emerging evidence suggests that 
representations across modalities converge toward shared 
geometric structures~\cite{huh2024position}, and that LLMs 
develop structured visual reasoning priors from relational 
patterns in text~\cite{han2025learningseeingdemystifyingllm}. 
These observations motivate our approach of aligning 
relational geometry, rather than token-level semantics 
alone, between the two modalities.

\vspace{-10pt}
\paragraph{Optimal transport and Gromov--Wasserstein Distance.}
OT is a classic topic in machine learning~\cite{ruschendorf1985wasserstein}. 
Cuturi’s entropically regularized Sinkhorn distance enables parallel and significantly faster approximations of discrete Wasserstein computation~\cite{cuturi2013sinkhorn}, inspiring a series of improved Sinkhorn‑type algorithms in subsequent work~\cite{lin2019efficient,altschuler2019massively,benamou2015iterative,altschuler2017near}.
An important variant of OT is the Gromov--Wasserstein (GW) 
distance~\cite{memoli2011gromov}, rooted in the 
Gromov--Hausdorff framework~\cite{gromov1999metric} and 
compares metric measure spaces by minimizing the 
distortion of pairwise distance structures under an 
optimal coupling, without requiring point-to-point 
correspondence. Efficient solvers include 
Frank--Wolfe 
methods~\cite{ferradans2014regularized,flamary2021pot}, 
entropic 
regularization~\cite{SolPeyKim16,sejourne2021unbalanced,rioux2024entropic}, 
and semidefinite 
relaxations~\cite{chen2024semidefinite}. 
Recent work has further improved GW along complementary directions, including scalable data-dependent optimization for large-scale GW computation and robust variants for structural noise~\cite{song2026lobcdgw,cheng2026achieving}.

\vspace{-10pt}
\section{Methodology: MIRROR via Semi-Inverse Gromov--Wasserstein}
\label{sec:method}
\label{sec:method_formulation}
\vspace{-5pt}
\begin{figure*}[tb]
  \centering
  \includegraphics[width=0.9\linewidth]{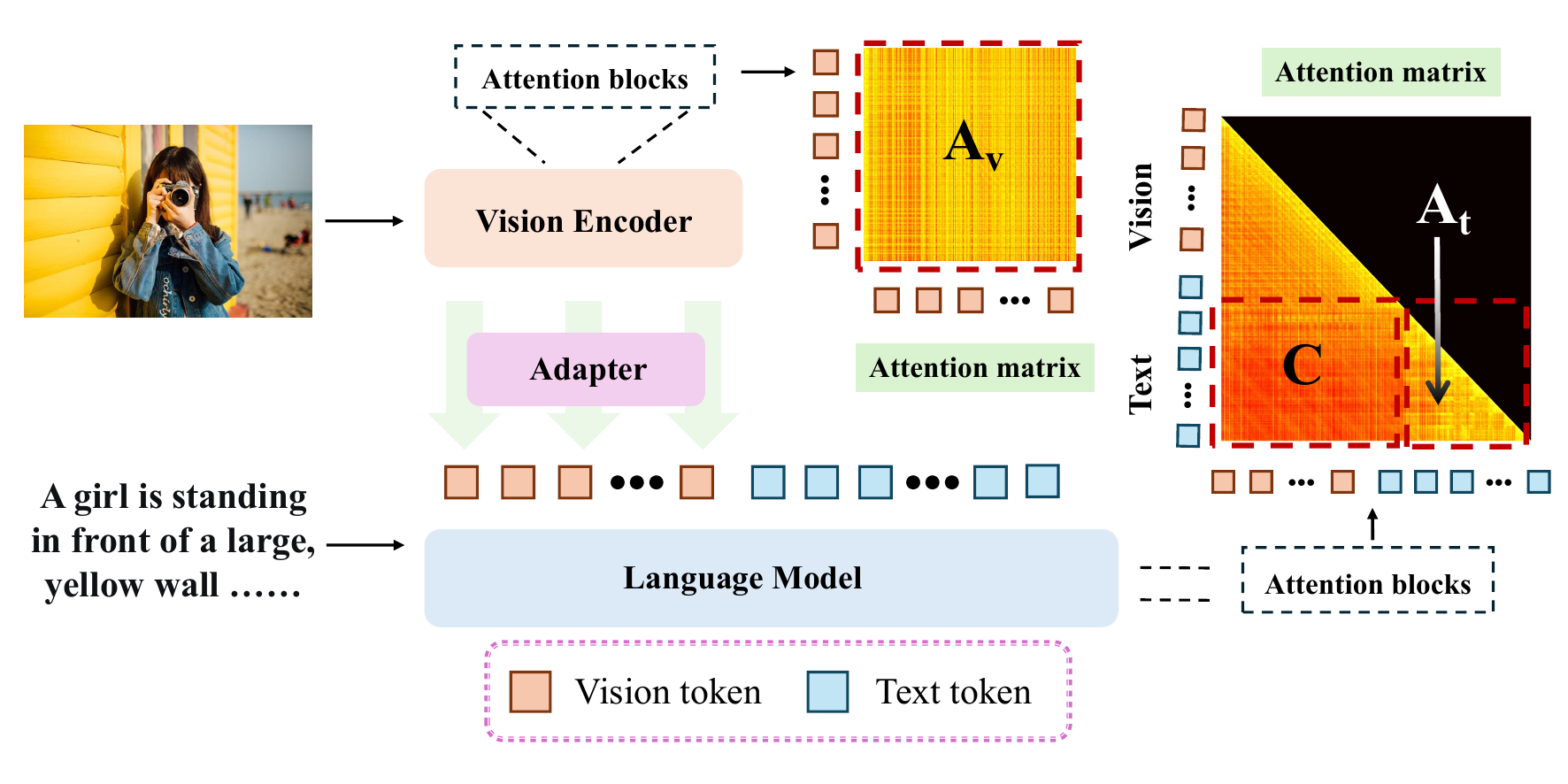}
  \vspace{-10pt}
  \caption{
    Overview of the MLLM architecture.
    An image is encoded into visual tokens via a vision encoder and adapter, then concatenated with text tokens and fed into the language model.
    $A_v$ denotes visual self-attention from the vision encoder, $C$ the cross-modal attention from text to visual tokens, and $A_t$ the causal text self-attention within the language model.
  }
  \label{fig:main_overview}
  \vspace{-18pt}
\end{figure*}

As illustrated in \cref{fig:main_overview}, a standard MLLM maps visual tokens into the LLM input space via a projection adapter $f:\mathcal{V}\rightarrow\mathbb{R}^d$ and is trained with the autoregressive objective 
\begin{equation}
\mathcal{L}_{\text{LM}}
=
-\sum_{t} \log p\!\left(w_t \mid f(v_1),\ldots,f(v_{n_v}),\, w_{<t}\right),
\label{eq:lm_objective}
\end{equation}
where $\{v_i\}_{i=1}^{n_v}$ are visual tokens from the vision encoder and $\{w_t\}$ are text tokens in the target sequence.
Optimizing \cref{eq:lm_objective} encourages identity-level alignment: each visual token is trained to be individually predictable from language context.
However, as shown in \cref{fig:motivation}, many MLLM failures on relational reasoning stem not from missing concept identities, but from missing relational structure.

We capture this intuition through a relational consistency requirement:
relational geometry should be preserved across modalities. Specifically, pairs of
concepts that are close (or far) in language space should exhibit a consistent
relational configuration in visual space.
Crucially, \cref{eq:lm_objective} provides no explicit supervision for such
relational consistency. Even when individual concepts transfer correctly,
their pairwise relations may remain geometrically inconsistent across modalities.

This observation motivates \textbf{MIRROR}, a geometric regularization framework that transfers relational priors from language representations to vision representations in MLLMs.
We derive the \textbf{Semi-Inverse Gromov--Wasserstein (SI-GW)} loss by posing and solving an inverse geometric alignment problem within the MLLM architecture (\cref{sec:sigw}),
describe three instantiation strategies that make SI-GW loss stable at scale (\cref{sec:hierarchical_design}),
and present the overall training procedure (\cref{sec:overall_objective}).

\vspace{-10pt}
\subsection{Semi-Inverse Gromov--Wasserstein (SI-GW) Loss}
\label{sec:sigw}

Enforcing the relational consistency requirement requires answering three questions:
\emph{(i)} how to measure relational geometry within each modality,
\emph{(ii)} how to establish correspondences between visual and text tokens, and
\emph{(iii)} given the language geometry and the correspondence, what the visual geometry \emph{should} look like. To address those problem, MIRROR augments the standard generation loss with a geometric regularization term:
\vspace{-5pt}
\begin{equation}
\mathcal{L}_{\text{total}}
=
\mathcal{L}_{\text{LM}}
+
\lambda\,\mathcal{L}_{\text{SI-GW}},
\label{eq:total_loss}
\vspace{-5pt}
\end{equation}
where $\lambda>0$ controls the strength of structural alignment.
The derivation of $\mathcal{L}_{\text{SI-GW}}$ is  divided in four steps:
\textit{Steps~1} and \textit{2} construct the necessary geometric ingredients from the MLLM's own attention maps.
The key contribution comes in \textit{Step~3}, where we pose the \emph{semi-inverse GW problem}, which asks what the ideal visual geometry is, and prove that it admits a closed-form solution.
\textit{Step~4} turns this solution into an efficient training loss. Below we elaborate on those steps.

\vspace{-7pt}
\paragraph{Step 1: Intra-modal geometry.}
Since relational alignment relies on pairwise structural relations, the primary requirement is an intra-modality distance metric, for which self-attention is a natural proxy: attention in language models is known to encode inter-token relations, from syntactic dependencies to knowledge-graph links~\cite{ren-etal-2024-identifying}, so high mutual attention between two tokens can be read as semantic proximity and low attention as distance.
Let $A_t\!\in\!\mathbb{R}_+^{n_t\times n_t}$ and $A_v\!\in\!\mathbb{R}_+^{n_v\times n_v}$ denote self-attention matrices extracted from the text branch and the vision encoder, respectively  (\cref{fig:main_overview}).
Raw self-attention is generally asymmetric and concentrated near zero, so we symmetrize via $A \leftarrow (A + A^\top)/2$ and apply a negative log-transform to amplify differences:

\vspace{-5pt}
\begin{equation}
D^t = -\log(A_t+\varepsilon),
\qquad
D^v = -\log(A_v+\varepsilon),
\label{eq:d_t_v}
\end{equation}
where $\varepsilon\in\mathbb R_+$ is a sufficient small amount that prevents numerical instability. Therefore, we obtain text-wise discrepancy $d_t(i,j)=D^t_{i,j},$ and vision-wise discrepancy $d_v(i,j)=D^v_{i,j}.$
Under this monotone transformation, semantically related token pair yields small discrepancy, while unrelated pair is mapped to large discrepancy.

These discrepancies quantify geometry {within} each modality, but enforcing relational consistency requires comparing across modalities, specifically, knowing which visual token corresponds to which text token.

\vspace{-7pt}
\paragraph{Step 2: Cross-modal correspondence.}
Comparing two geometries defined on different index sets requires a coupling---a soft assignment that maps tokens in one space to tokens in the other.
In our setting, such a correspondence need not be estimated externally: the text-to-vision cross-attention at an intermediate LLM layer already provides a data-dependent soft assignment.
Concretely, let $C\!\in\!\mathbb{R}^{n_t\times n_v}$ denote the (aggregated) cross-attention matrix, whose entry $C_{ij}$ measures how strongly text token $i$ attends to visual token $j$.

With the language geometry $D^t$, the visual geometry $D^v$, and the coupling $C$ all in hand, the relational consistency requirement can now be stated precisely: \emph{the visual distance structure $D^v$ should be consistent with $D^t$ under the correspondence $C$}.
The question that remains is what ``consistent'' means quantitatively, so that $\mathcal{L}_{\text{SI-GW}}$ can be derived.

\vspace{-7pt}
\paragraph{Step 3: The semi-inverse GW problem.}
We formalize the consistency requirement above as an optimization problem.
Treating the language geometry $D^t$ as a fixed teacher and the coupling $C$ as a frozen correspondence, we ask: \emph{what is the visual geometry $D^v$ that best preserves the relational structure of $D^t$ under $C$?}
We call this the \textbf{semi-inverse GW problem}:
\vspace{-5pt}
\begin{equation}
\widehat{D}^v
\;=\;
\argmin_{D^v \in \mathbb{R}^{n_v\times n_v}}
\sum_{i,j=1}^{n_t}\sum_{k,\ell=1}^{n_v}
\bigl|d_t(i,j)-d_v(k,\ell)\bigr|^2\, C_{ik}\,C_{j\ell}.
\label{eq:sigw_problem}
\vspace{-5pt}
\end{equation}
The objective is derived from the Gromov--Wasserstein problem (\cref{eq:gw_distance}) in which both $D^t$ and $C$ are held fixed and only $D^v$ is optimized. The name ``semi-inverse'' is because the roles of the two spaces are asymmetric: the language side specifies the target structure, and the visual side must conform.
 
\begin{theorem}[Closed-form solution of the semi-inverse GW problem]
\label{thm:inverse_gw}
Let $D^t\!\in\!\mathbb{R}^{n_t \times n_t}$ and $C\!\in\!\mathbb{R}^{n_t \times n_v}$ with strictly positive column marginals $b = C^\top\mathbf{1}_{n_t}\!>\!0$.
The semi-inverse GW problem (\cref{eq:sigw_problem}) has the unique minimizer
\begin{equation}
\widehat{D}^v
=
\bigl(C^\top D^t\, C\bigr)\oslash \bigl(b\,b^\top\bigr),
\label{eq:dhat}
\end{equation}
where $\oslash$ denotes entrywise division.
\end{theorem}

\begin{proof}[Sketch]
Expanding \cref{eq:sigw_problem} and collecting terms in $d_v(k,\ell)$ yields a  weighted least-squares problem in each entry.
The first-order optimality condition gives $\widehat{D}^v_{k,\ell} = \frac{(C^\top D^t\, C)_{k\ell}}{b_k b_\ell}$, which is the unique minimizer since the objective is strictly convex in $D^v$.
The full derivation is provided in Appendix~B. \qed
\end{proof}
\textbf{How to interpret $\widehat{D}^v$? } 
$\widehat{D}^v$ serves as the coupling-induced {target vision distance matrix}.
This target admits an intuitive interpretation as a conditional expectation. For any fixed visual token pair $(k,\ell)$, we use the coupling weights to induce a probability distribution over text token pairs $(i,j)$:
\vspace{-5pt}
\begin{equation}
p(i,j \mid k,\ell)
\;\triangleq\;
\frac{C_{ik}\,C_{j\ell}}{b_k\,b_\ell},
\qquad
\sum_{i=1}^{n_t}\sum_{j=1}^{n_t} p(i,j\mid k,\ell)=1,
\label{eq:cond_pair}
\vspace{-5pt}
\end{equation}
where the normalization condition follows from $\sum_i C_{ik}=b_k$ and $\sum_j C_{j\ell}=b_\ell$.
Under this conditional distribution, $\widehat{D}^v$ admits the following expectation form:
\vspace{-5pt}
\begin{equation}
\widehat{D}^v_{k,\ell}
=
\sum_{i=1}^{n_t}\sum_{j=1}^{n_t}
d_t(i,j)\,p(i,j\mid k,\ell)
=
\mathbb{E}_{(i,j)\sim p(\cdot,\cdot\mid k,\ell)}\!\left[d_t(i,j)\right],
\label{eq:ctg_expectation}
\vspace{-5pt}
\end{equation}
\ie, the entry $\widehat{D}^v_{k,\ell}$ is the expected text-space distance between the text tokens that attend to visual tokens $k$ and $\ell$, respectively.
In other words, $\widehat{D}^v$ describes the geometric configuration of visual representation that would emerge if all language relational structure were faithfully inherited through the coupling, which is the precise answer to the question posed at the beginning of this subsection.





\vspace{-7pt}
\paragraph{Step 4: Training loss.}
With $\widehat{D}^v$ as a concrete target, we define the SI-GW loss as the squared Frobenius discrepancy between the current visual geometry and the ideal one:
\begin{definition}[Semi-Inverse Gromov--Wasserstein loss]
\label{def:sigw}
\begin{equation}
\boxed{
\mathcal{L}_{\text{SI-GW}}
=
\left\|D^v - \widehat{D}^v\right\|_F^2
=
\left\| D^v - \frac{C^\top D^t\, C}{\,b\,b^\top\,}\right\|_F^2.
}
\label{eq:sigw_loss}
\end{equation}
\end{definition}

\noindent By \cref{thm:inverse_gw}, minimizing $\mathcal{L}_{\text{SI-GW}}$ over $D^v$ has the same effect as minimizing the quartic semi-inverse objective in \cref{eq:sigw_problem}
The critical advantage lies in computational efficiency: whereas direct evaluation of \cref{eq:sigw_problem} requires $O(n_t^2 n_v^2)$ operations from the sum over all pairs of token pairs, the closed-form reduction to \cref{eq:sigw_loss} admits $O(n_t^2 n_v + n_v^2 n_t)$ evaluation.
Specifically, in \cref{eq:sigw_loss}, computing the denominator  term $b\!=\!C^\top\mathbf{1}_{n_t}$ costs $O(n_t n_v)$; the dominant term $C^\top D^t\, C$ requires $O(n_t^2 n_v + n_v^2 n_t)$ via two matrix multiplications; and the entrywise division and Frobenius norm each cost $O(n_v^2)$.

\vspace{-10pt}
\subsection{Extracting Stable Geometric Ingredients}
\label{sec:hierarchical_design}


Computing the SI-GW loss (\cref{eq:sigw_loss}) requires extracting
$D^t$, $D^v$, and $C$ from the multi-layer, multi-head Transformer.
Straightforward extraction, e.g., from a single layer
with uniform head averaging,
yields noisy couplings and unstable gradients in practice
(see \cref{sec:analysis} for empirical evidence).
As shown in \cref{fig:extraction_strategies}, we introduce three complementary strategies, each targeting a specific source of instability; the systematic ablations are reported in \cref{sec:analysis}.

\vspace{-10pt}
\subsubsection{Layer Decoupling for Stable Geometry Transfer}
\label{sec:layer_selection}

\begin{wrapfigure}{r}{0.6\textwidth}
  \centering
  \vspace{-12pt}
  \includegraphics[width=\linewidth]{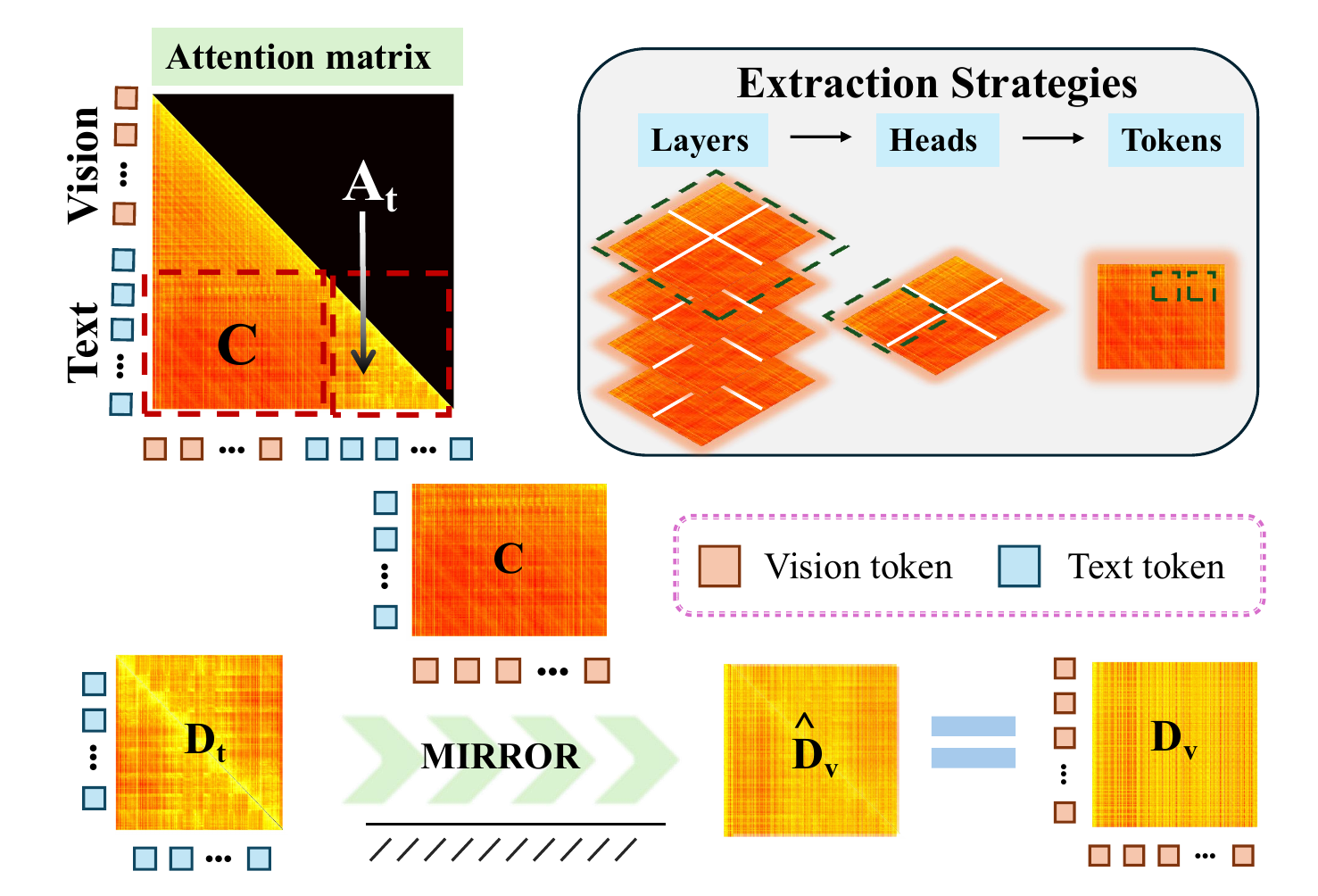}
  \caption{
    MIRROR extraction and alignment.
    $A_t$ and $C$ are extracted from separate LLM layers, and $A_v$ from the final ViT layer, with selection applied at the layer, head, and token levels.
    MIRROR constructs the target geometry $\widehat{D}^v$ from $D^t$ and $C$, and drives $D^v$ toward it via the SI-GW loss.
  }
  \label{fig:extraction_strategies}
  \vspace{-12pt}
\end{wrapfigure}

In decoder-only MLLMs, $A_t$ and $C$ are sub-matrices of
the same causal attention map, jointly normalized by a single
softmax. Extracting both from the same layer couples their
gradients and destabilizes training.
We resolve this by \emph{layer decoupling}, extracting
each ingredient from a separate stage best suited to its role:
$A_t$ is extracted from an early LLM layer, where textual relational
geometry is relatively pure before strong multimodal mixing;
$A_v$ from the final ViT layer, which most directly governs
downstream visual reasoning;
and $C$ from an intermediate LLM layer, where representations
balance semantic richness and
concentration~\cite{tao2024probing,viswanathan2025geometrytokensinternalrepresentations}.

In practice, the full causal attention map at each LLM layer
spans all token types. We separate the text--text and
text--vision sub-blocks via the padding mask to yield $A_t$
and the raw cross-attention, respectively.

\vspace{-10pt}
\subsubsection{Selective Head Aggregation for Reliable Coupling}
\label{sec:head_selection}
At the coupling layer, multi-head attention yields $H$
cross-attention maps
$\{A^h\}_{h=1}^{H}$.
Prior work has shown that attention heads exhibit highly
uneven importance~\cite{voita2020analyzing,michel2019sixteen};
many produce near-uniform distributions that dilute sharp
correspondences when naively averaged.
To identify the most informative heads, we compute the entropy
$\mathcal{H}_h$ of each head $h$'s cross-attention map $A^h$:
\begin{equation}
\mathcal{H}_h = -\sum_{i,j} A^h_{ij}\log A^h_{ij},
\label{eq:head_entropy}
\end{equation}
where lower entropy indicates a more concentrated, and thus
more discriminative, cross-modal assignment.
We retain the top-$k$ lowest-entropy heads and form the coupling
by averaging over them:
\vspace{-5pt}
\begin{equation}
C = \frac{1}{k}\sum_{h\in \mathcal{S}_k} A^h,
\qquad
\mathcal{S}_k
= \bigl\{h : \mathcal{H}_h \text{ is among the } k
\text{ smallest}\bigr\}.
\label{eq:head_agg}
\vspace{-5pt}
\end{equation}

\vspace{-5pt}
\subsubsection{Non-semantic Token Filtering for Cleaner Text Geometry}
\label{sec:token_filtering}

The text stream contains tokens that carry no visual
semantics, including system prompts, image placeholders
(\eg, \texttt{<image>}), and formatting markers.
These tokens appear in $A_t$ yet do not participate in
meaningful relational structure, introducing noise into
the derived distance matrix $D^t$.
To mitigate this, we apply a soft, position-aware suppression.
For each text token $j$, we define a weight
$w_j \in [0,1]$ based on its position and its coupling mass
$\sum_k C_{jk}$: tokens that appear early in the sequence
(where system prompts typically reside) and attend weakly
to visual tokens receive large $w_j$.
The attention entries are then attenuated as:
\vspace{-5pt}
\begin{equation}
\widetilde{A}_t(i,j)
=
\bigl(1 - \alpha \cdot w_j\bigr)\, A_t(i,j),
\label{eq:token_suppress}
\vspace{-5pt}
\end{equation}
where $\alpha \in [0,1]$ controls the overall suppression
strength, followed by row-wise re-normalization before
computing $D^t$ via \cref{eq:d_t_v}.
Compared with hard removal, this soft filtering preserves
potentially useful context while reducing the influence of
non-semantic tokens.
Details on the weighting scheme are provided in Appendix~C.

\vspace{-10pt}
\subsection{Training Procedure}
\label{sec:overall_objective}

\renewcommand{\algorithmicrequire}{\textbf{Input:}}
\renewcommand{\algorithmicensure}{\textbf{Output:}}

\begin{algorithm}[t]
\caption{MIRROR Training}
\label{alg:MIRROR}
\begin{algorithmic}[1]
\REQUIRE Training data $\mathcal{D}$, MLLM parameters $\theta$, SI-GW weight $\lambda$, top-$k$ heads, stability constant $\varepsilon$, learning rate $\eta$
\ENSURE Fine-tuned parameters $\theta^*$
\FOR{each mini-batch $(\mathbf{x}, \mathbf{y}) \sim \mathcal{D}$}
    \STATE Perform forward pass through ViT and LLM
    \STATE {\color{cyan!70!blue}\textit{// Decoupled attention extraction}}
    \STATE Extract $A_v$ (final ViT), $A_t$ (early LLM), $\{A^h\}_{h=1}^{H}$ (intermediate LLM)
    \STATE Symmetrize: $A_v \gets (A_v + A_v^\top)/2$, \; $A_t \gets (A_t + A_t^\top)/2$
    \STATE {\color{cyan!70!blue}\textit{// Entropy-based head selection and coupling}}
    \STATE Compute per-head entropy: $\mathcal{H}_h \gets -\sum_{i,j} A^h_{ij} \log A^h_{ij}$
    \STATE Select $\mathcal{S}_k \gets \{h : \mathcal{H}_h \text{ is among the } k \text{ smallest}\}$
    \STATE Compute coupling matrix: $C \gets \frac{1}{k}\sum_{h \in \mathcal{S}_k} A^h$
    \STATE {\color{cyan!70!blue}\textit{// Non-semantic token filtering}}
    \STATE Filter non-semantic tokens in $A_t$ to obtain $\tilde{A}_t$ \quad (\cref{eq:token_suppress})
    \STATE {\color{cyan!70!blue}\textit{// Semi-Inverse Gromov--Wasserstein loss}}
    \STATE $D^t \gets -\log(\tilde{A}_t + \varepsilon)$, \quad $D^v \gets -\log(A_v + \varepsilon)$
    \STATE $\mathbf{b} \gets C^\top \mathbf{1}_{n_t}$, \quad $\widehat{D}^v \gets (C^\top D^t\, C) \oslash (\mathbf{b}\,\mathbf{b}^\top)$
    \STATE $\mathcal{L}_{\mathrm{SI\text{-}GW}} \gets \lVert D^v - \widehat{D}^v \rVert_F^2$
    \STATE {\color{cyan!70!blue}\textit{// Parameter update}}
    \STATE $\mathcal{L}_{\mathrm{total}} \gets \mathcal{L}_{\mathrm{LM}} + \lambda \cdot \mathcal{L}_{\mathrm{SI\text{-}GW}}$
    \STATE $\theta \gets \theta - \eta\,\nabla_\theta \mathcal{L}_{\mathrm{total}}$
\ENDFOR
\RETURN $\theta^* \gets \theta$
\end{algorithmic}
\end{algorithm}

The complete MIRROR training procedure is summarized in \cref{alg:MIRROR}.
At each training step, the standard forward pass through the ViT and LLM produces all required attention matrices (line~2).
Layer decoupling, head selection, and token suppression (\cref{sec:hierarchical_design}) are then applied to extract clean $D^t$, $D^v$, and $C$ (lines~3--6), from which $\widehat{D}^v$ and the SI-GW loss are computed (lines~7--10).
The combined objective (\cref{eq:total_loss}) drives parameter updates (line~11).

\vspace{-10pt}
\paragraph{No additional inference overhead.}
MIRROR introduces no additional parameters and inference cost. 
All SI-GW loss computations use attention matrices already produced during the forward pass and are active only during training.
At test time, the model architecture and computational cost remain identical to the baseline MLLM.

\vspace{-10pt}
\section{Experiments}
\label{sec:experiment}
\vspace{-5pt}
We first describe the implementation details (\cref{sec:impl_details}) and evaluation benchmarks (\cref{sec:benchmarks}), then present main results (\cref{sec:main_results}) followed by ablation and design analysis (\cref{sec:analysis}).
\vspace{-10pt}
\subsection{Implementation Details}
\label{sec:impl_details}
 \vspace{-5pt}
\noindent\textbf{Model and Training Configuration.}
We conduct experiments on two representative MLLM families, LLaVA-1.5~\cite{liu2024improved} and LLaVA-NeXT~\cite{liu2024llavanext}, each at the 7B and 13B scales. Both families are selected for (1) fully open-sourced training data that enables rigorous reproduction and evaluation, and (2) clean architectural designs that eliminate confounding factors from complex structures. For each model, we perform MIRROR fine-tuning for 1 epoch on 8 NVIDIA RTX 6000 Ada GPUs (48GB), following the original training protocol of each family. All evaluations use greedy decoding, and baseline results are reproduced from official checkpoints.

\noindent\textbf{MIRROR Hyperparameter Settings.}
The SI-GW loss weight $\lambda$ in \cref{eq:total_loss} is set
to 0.002. For layer decoupling (\cref{sec:layer_selection}),
we extract the text self-attention from Layer~1 of the LLM,
the visual self-attention from the final vision encoder layer,
and the coupling matrix $C$ from Layer~16. For head selection
(\cref{eq:head_agg}), we retain the $k{=}8$ lowest-entropy heads. For non-semantic token filtering
(\cref{eq:token_suppress}), we set the initial attenuation strength $\alpha=0.5$.
\vspace{-10pt}
\subsection{Evaluation Benchmarks}
\label{sec:benchmarks}
\vspace{-5pt}
We evaluate our model on two categories of tasks:

\noindent\textbf{Relational Reasoning Tasks.}
GQA~\cite{hudson2019gqa} is a large-scale scene graph reasoning dataset requiring multi-hop reasoning and spatial relation understanding. It provides fine-grained annotations across five semantic categories: \textit{Attribute}, \textit{Category}, \textit{Global}, \textit{Object}, and \textit{Relation}.
BLINK~\cite{fu2024blink} evaluates fine-grained visual perception and multimodal reasoning, spanning three difficulty levels: \textit{low-level} pattern matching (visual correspondence, reflectance), \textit{mid-level} spatial reasoning (spatial relations, multi-view, jigsaw), and \textit{high-level} understanding (localization, counting, forensics, similarity).

\noindent\textbf{General Vision-Language Tasks.}
VQAv2~\cite{balanced_vqa_v2} is a widely-used benchmark covering diverse everyday scenarios.
POPE~\cite{Li-hallucination-2023} evaluates hallucination issues by testing object existence judgment.
RealWorldQA~\cite{grok15v2025} contains real-world commonsense reasoning questions.

\vspace{-10pt}
\subsection{Main Results}
\label{sec:main_results}
Our main experiments address three questions:

\noindent\textbf{RQ1.} Can MIRROR effectively enhance relational structure understanding in multimodal models? 

\noindent\textbf{RQ2.} Does it generalize across reasoning levels, from spatial to semantic relations? 

\noindent\textbf{RQ3.} Can these benefits be achieved without degrading general vision-language performance? 


\begin{table}[tb]
  \caption{\textbf{Structured relational reasoning on GQA.}
    \colorbox{lightgray}{\strut Shaded rows} denote MIRROR-enhanced models.
    \textbf{Bold} highlights notable gains ($\Delta \geq 1.0$).}
  \label{tab:gqa_results}
  \centering
  \small
  \setlength{\tabcolsep}{3.5pt}
  \resizebox{\linewidth}{!}{%
  \vspace{-7pt}
  \begin{tabular}{@{}lcccccc@{}}
    \toprule
    Method & Attr. & Categ. & Global & Obj. & Rel. & Overall \\
    \midrule
    LLaVA-1.5-7B
      & 67.52 & 51.17 & 62.66 & 87.79 & 53.15 & 61.16 \\
    \rowcolor{lightgray}
    \quad + MIRROR
      & 68.08\sss{+0.56} & \bd{53.61}\sss{+2.44} & \bd{65.82}\sss{+3.16}
      & 88.05\sss{+0.26} & 53.78\sss{+0.63} & 61.93\sss{+0.77} \\
    \addlinespace[4pt]
    LLaVA-1.5-13B
      & 68.85 & 53.79 & 63.29 & 88.30 & 54.30 & 62.47 \\
    \rowcolor{lightgray}
    \quad + MIRROR
      & 69.31\sss{+0.46} & 54.57\sss{+0.78} & \bd{64.33}\sss{+1.04}
      & 88.56\sss{+0.26} & 54.43\sss{+0.13} & 62.81\sss{+0.34} \\
    \midrule
    LLaVA-NeXT-7B
      & 71.18 & 56.73 & 69.34 & 89.42 & 57.38 & 64.74 \\
    \rowcolor{lightgray}
    \quad + MIRROR
      & 71.63\sss{+0.45} & 57.49\sss{+0.76} & \bd{71.47}\sss{+2.13}
      & 89.61\sss{+0.19} & \bd{58.82}\sss{+1.44} & 65.56\sss{+0.82} \\
    \addlinespace[4pt]
    LLaVA-NeXT-13B
      & 71.92 & 57.51 & 70.17 & 89.85 & 58.14 & 65.43 \\
    \rowcolor{lightgray}
    \quad + MIRROR
      & 72.27\sss{+0.35} & 58.14\sss{+0.63} & \bd{71.68}\sss{+1.51}
      & 90.01\sss{+0.16} & \bd{59.21}\sss{+1.07} & 66.05\sss{+0.62} \\
    \bottomrule
  \end{tabular}%
  }
  \vspace{-10pt}
\end{table}

\vspace{-10pt}
\subsubsection{Structured Relational Reasoning (RQ1)}
\label{sec:rq1}

We evaluate MIRROR on structured relational reasoning using GQA. As shown in \cref{tab:gqa_results}, MIRROR yields steady overall gains across all four model configurations: +0.77\% and +0.34\% for LLaVA-1.5 (7B/13B), and +0.82\% and +0.62\% for LLaVA-NeXT (7B/13B), suggesting that its effectiveness generalizes across model scales and architectures.

A cross-category analysis reveals a coherent pattern across all four configurations. The largest gains appear in \textbf{Global} tasks (+1.04\% to +3.16\%), which require integrating relationships among multiple objects (\eg, ``how many red objects are on the table''). \textbf{Category} tasks involving fine-grained discrimination also benefit (+0.63\% to +2.44\%). Interestingly, \textbf{Relation} tasks show more pronounced improvements on the stronger LLaVA-NeXT baselines (7B: +1.44\%, 13B: +1.07\%) than on LLaVA-1.5 (7B: +0.63\%, 13B: +0.13\%). This suggests that a better-aligned feature space allows the model to more effectively leverage its relational reasoning capacity. In contrast, \textbf{Object} and \textbf{Attribute} tasks exhibit only modest gains, as expected: these categories depend more on local visual feature recognition than on inter-concept relational structures.

This consistent pattern across four model configurations, where Global and Relation tasks improve notably while Object tasks remain near-baseline, supports the hypothesis that MIRROR primarily operates at the level of relational geometry, addressing the cross-modal misalignment formalized in \cref{sec:method_formulation} rather than uniformly enhancing feature quality.
Appendix~D visualizes this restructuring, showing how MIRROR drives the visual geometry $D^v$ toward the target $\widehat{D}^v$.
\vspace{-10pt}

\begin{table*}[tb]
  \caption{\textbf{Cross-level reasoning generalization on BLINK.}
    Tasks are grouped by reasoning level: \textbf{L}ow-level visual,
    \textbf{M}id-level spatial, and \textbf{H}igh-level semantic.
    \colorbox{lightgray}{\strut Shaded rows} denote MIRROR-enhanced models;
    $\Delta$ rows show absolute improvement; ``--'' indicates no change.
    \textbf{Bold} highlights notable gains ($\Delta \geq 3.0$) and corresponding scores.}
  \label{tab:blink_results}
  \centering
  \small
  \setlength{\tabcolsep}{3.5pt}
  \resizebox{\textwidth}{!}{%
  \vspace{-7pt}
  \begin{tabular}{@{}l ccc cccc ccccccc c@{}}
    \toprule
    \multirow{2.5}{*}{Method}
      & \multicolumn{3}{c}{Low-level}
      & \multicolumn{4}{c}{Mid-level}
      & \multicolumn{7}{c}{High-level}
      & \multirow{2.5}{*}{Avg.} \\
    \cmidrule(lr){2-4} \cmidrule(lr){5-8} \cmidrule(lr){9-15}
      & \romark{Vis.} & \romark{Refl.} & \romark{Dep.}
      & \romark{Spat.} & \romark{M-v.} & \romark{Jig.} & \romark{Art}
      & \romark{Loc.} & \romark{Cnt.} & \romark{For.} & \romark{IQ}
      & \romark{Sim.} & \romark{Sem.} & \romark{Func.}
      & \\
    \midrule
    LLaVA-1.5-7B
      & 25.6 & 26.9 & 51.6
      & 59.4 & 44.4 & 52.7 & 47.5
      & 55.7 & 44.2 & 25.0 & 21.3 & 47.4 & 30.9 & 33.9
      & 40.5 \\
    \rowcolor{lightgray}
    \quad + MIRROR
      & \bd{30.8} & 26.9 & \bd{54.8}
      & \bd{64.3} & \bd{48.1} & 52.7 & 47.5
      & \bd{60.7} & \bd{48.3} & 25.0 & \bd{24.7} & 47.4 & \bd{36.0} & 33.9
      & 42.9 \\
    \quad\quad $\Delta$
      & \bd{\gn{+5.2}} & -- & \bd{\gn{+3.2}}
      & \bd{\gn{+4.9}} & \bd{\gn{+3.8}} & -- & --
      & \bd{\gn{+4.9}} & \bd{\gn{+4.2}} & -- & \bd{\gn{+3.3}} & -- & \bd{\gn{+5.0}} & --
      & \gn{+2.5} \\
    \addlinespace[4pt]
    LLaVA-1.5-13B
      & 23.8 & 42.5 & 54.8
      & 67.1 & 44.4 & 52.7 & 47.5
      & 42.6 & 44.2 & 28.0 & 22.0 & 47.4 & 33.8 & 28.5
      & 41.4 \\
    \rowcolor{lightgray}
    \quad + MIRROR
      & 25.0 & 42.5 & 54.8
      & 69.2 & 44.4 & 52.7 & 47.5
      & 43.4 & 46.7 & \bd{32.6} & \bd{26.7} & 47.4 & 36.7 & 29.2
      & 42.8 \\
    \quad\quad $\Delta$
      & \gn{+1.2} & -- & --
      & \gn{+2.1} & -- & -- & --
      & \gn{+0.8} & \gn{+2.5} & \bd{\gn{+4.6}} & \bd{\gn{+4.7}} & -- & \gn{+2.9} & \gn{+0.8}
      & \gn{+1.4} \\
    \midrule
    LLaVA-NeXT-7B
      &  8.7 & 40.3 & 54.8
      & 69.9 & 44.4 & 50.0 & 47.0
      & 31.1 & 46.7 & 20.5 & 17.3 & 46.7 & 29.5 & 18.5
      & 36.9 \\
    \rowcolor{lightgray}
    \quad + MIRROR
      & \bd{18.6} & 40.3 & 55.7
      & 69.9 & 44.4 & 52.0 & 48.7
      & 31.1 & 47.5 & \bd{23.5} & \bd{22.0} & 48.2 & 30.2 & \bd{23.9}
      & 39.2 \\
    \quad\quad $\Delta$
      & \bd{\gn{+9.9}} & -- & \gn{+0.8}
      & -- & -- & \gn{+2.0} & \gn{+1.7}
      & -- & \gn{+0.8} & \bd{\gn{+3.0}} & \bd{\gn{+4.7}} & \gn{+1.5} & \gn{+0.7} & \bd{\gn{+5.4}}
      & \gn{+2.4} \\
    \addlinespace[4pt]
    LLaVA-NeXT-13B
      & 14.0 & 39.6 & 48.4
      & 69.2 & 46.6 & 47.3 & 46.2
      & 54.9 & 49.2 & 24.2 &  2.0 & 47.4 & 16.5 &  6.9
      & 35.8 \\
    \rowcolor{lightgray}
    \quad + MIRROR
      & 14.0 & 41.8 & 48.4
      & 69.2 & \bd{50.4} & \bd{53.3} & 46.2
      & \bd{59.8} & 49.2 & \bd{28.8} &  2.7 & 47.4 & \bd{24.5} & \bd{12.3}
      & 38.3 \\
    \quad\quad $\Delta$
      & -- & \gn{+2.2} & --
      & -- & \bd{\gn{+3.8}} & \bd{\gn{+6.0}} & --
      & \bd{\gn{+4.9}} & -- & \bd{\gn{+4.6}} & \gn{+0.7} & -- & \bd{\gn{+7.9}} & \bd{\gn{+5.4}}
      & \gn{+2.5} \\
    \bottomrule
  \end{tabular}%
  }
\end{table*}

\subsubsection{Cross-Level Reasoning Generalization (RQ2)}
\label{sec:rq2}

We next evaluate MIRROR on BLINK to assess whether the improvements generalize across reasoning levels.
As shown in \cref{tab:blink_results}, all four configurations achieve consistent overall gains (+1.4\% to +2.5\%), with clear patterns across BLINK's difficulty spectrum.

\textbf{Mid-level spatial reasoning} tasks show the most consistent improvements, particularly spatial relation understanding (LLaVA-1.5-7B: +4.9\%) and multi-view reasoning (+3.8\%).
These tasks require modeling relative positions and spatial configurations between objects, which directly aligns with the relational geometry strengthened by structural alignment.
\textbf{High-level semantic reasoning} tasks also benefit substantially: localization (+4.9\%), counting (+4.2\%), semantic correspondence (+5.0\% on LLaVA-1.5-7B, +7.9\% on LLaVA-NeXT-13B), and functional reasoning (+5.4\% on both LLaVA-NeXT variants).
Tasks that depend more on holistic perception (\eg, image similarity) remain largely stable.
In contrast, \textbf{low-level visual tasks} such as reflection and depth estimation yield mixed results. This is expected, as these tasks primarily rely on low-level feature extraction rather than the inter-concept relational structure encoded in language.

Overall, these results indicate that MIRROR generalizes across reasoning levels, with gains concentrated in mid- and high-level tasks that involve spatial and semantic correspondences, while low-level visual tasks remain largely unaffected.
\vspace{-20pt}
\subsubsection{General Vision-Language Performance (RQ3)}
\label{sec:rq3}
Finally, we evaluate general VQA tasks to verify that MIRROR preserves overall vision-language understanding. As shown in \cref{tab:main_comparison}, MIRROR maintains or improves performance across all four model configurations.

On VQAv2, LLaVA-1.5-7B achieves a +1.8 gain, while LLaVA-NeXT remains stable (7B: $-$0.1, 13B: +0.1). On POPE, all configurations improve (+0.3 to +0.6), suggesting that better cross-modal distance structures help reduce object hallucination. RealWorldQA remains stable throughout. Notably, LLaVA-NeXT with MIRROR achieves the best results within each scale on both GQA and POPE, suggesting that structural alignment can complement architectural improvements.

\begin{table}[tb]
  \caption{\textbf{General vision-language performance on standard benchmarks.}
    MIRROR preserves or improves general VQA while enhancing relational reasoning.
    Baseline results from~\cite{liu2024improved} and respective publications.
    \colorbox{lightgray}{\strut Shaded rows} denote MIRROR-enhanced models.
    \textbf{Bold} = best within same scale.}
  \label{tab:main_comparison}
  \centering
  \small
  \setlength{\tabcolsep}{5pt}
  \vspace{-7pt}
  \begin{tabular}{@{}llcccc@{}}
    \toprule
    Method & LLM & VQAv2 & GQA & POPE & RWQA \\
    \midrule
    \multicolumn{6}{@{}l}{\textit{7B-class models}} \\
    \addlinespace[2pt]
    InstructBLIP-7B~\cite{dai2024instructblip}
      & Vicuna-7B & --\rlap{$^\text{h}$} & 49.2 & 74.4 & -- \\
    Qwen-VL-Chat~\cite{bai2023qwen}
      & Qwen-7B & 78.2\rlap{$^\dagger$} & 57.5\rlap{$^\dagger$} & -- & -- \\
    LLaVA-1.5-7B~\cite{liu2024improved}
      & Vicuna-7B & 79.3 & 61.2 & 85.7 & 56.6 \\
    \rowcolor{lightgray}
    \quad + MIRROR
      & Vicuna-7B & 81.1 & 61.9 & 86.3 & 56.8 \\
    LLaVA-NeXT-7B~\cite{liu2024llavanext}
      & Mistral-7B & \bd{82.2} & 64.8 & 86.7 & \bd{57.8} \\
    \rowcolor{lightgray}
    \quad + MIRROR
      & Mistral-7B & 82.1 & \bd{65.5} & \bd{87.1} & 57.7 \\
    \midrule
    \multicolumn{6}{@{}l}{\textit{13B-class models}} \\
    \addlinespace[2pt]
    BLIP-2~\cite{li2023blip}
      & Vicuna-13B & 65.0 & 41.0 & 85.3 & -- \\
    InstructBLIP-13B~\cite{dai2024instructblip}
      & Vicuna-13B & --\rlap{$^\text{h}$} & 49.5 & 78.9 & -- \\
    LLaVA-1.5-13B~\cite{liu2024improved}
      & Vicuna-13B & 81.3 & 62.5 & 85.5 & 56.1 \\
    \rowcolor{lightgray}
    \quad + MIRROR
      & Vicuna-13B & 81.7 & 62.8 & 86.0 & \bd{56.4} \\
    LLaVA-NeXT-13B~\cite{liu2024llavanext}
      & Vicuna-13B & 82.8 & 65.4 & 86.2 & 56.2 \\
    \rowcolor{lightgray}
    \quad + MIRROR
      & Vicuna-13B & \bd{82.9} & \bd{66.1} & \bd{86.5} & 56.1 \\
    \bottomrule
  \end{tabular}

  \smallskip
  {\footnotesize\raggedright
   $^\dagger$\,Eval-set images observed during pre-training.\quad
   $^\text{h}$\,Held-in during instruction tuning; zero-shot score not applicable.\quad
   RWQA\,=\,RealWorldQA.\par}
   \vspace{-17pt}
\end{table}

\vspace{-10pt}
\subsection{Ablation and Design Analysis}
\label{sec:analysis}
To understand the efficacy of MIRROR and its underlying mechanisms, we conduct comprehensive ablations and validate our specific design choices.

\vspace{-10pt}
\subsubsection{Component Ablation}
\label{sec:component_ablation}

\cref{tab:overall_ablation} examines the contribution of each
component via progressive integration and leave-one-out removal.
Cumulatively, layer decoupling and head selection together
improve the baseline, and token filtering provides a further gain.
The leave-one-out results reveal a clear hierarchy of importance.
Layer decoupling is indispensable: removing it causes training
to diverge, as the SI-GW gradients corrupt the model's attention
distributions.
Removing head selection also leads to severe degradation, confirming that noisy cross-modal couplings undermine the geometric signal.

\begin{table}[tb]
  \centering
  \small
  \caption{\textbf{Component ablation.}
    Top: cumulative addition. Bottom: leave-one-out.
    $\dagger$: training diverges.}
  \label{tab:overall_ablation}
  \setlength{\tabcolsep}{4pt}
  \vspace{-7pt}
  \begin{tabular}{@{}ccc cc@{}}
    \toprule
    \makecell{Layer\\Decoup.} & \makecell{Head\\Select.} & \makecell{Token\\Filter.}
      & GQA & BLINK \\
    \midrule
     & & 
      & 61.2 & 40.5 \\
    \multicolumn{5}{@{}l}{\textit{Progressive integration}} \\
    \checkmark & \checkmark & 
      & 61.7\sss{+0.5} & 41.9\sss{+1.4} \\
    \rowcolor{lightgray}
    \checkmark & \checkmark & \checkmark
      & \textbf{61.9}\sss{+0.2} & \textbf{42.9}\sss{+1.0} \\
    \midrule
    \multicolumn{5}{@{}l}{\textit{Leave-one-out}} \\
    \xmark &  & 
      & \multicolumn{2}{c}{diverged$\dagger$} \\
     & \xmark & 
      & 59.8\sss{-2.1} & 35.6\sss{-7.3} \\
    \bottomrule
  \end{tabular}
\end{table}

\subsubsection{Design Choice Validation}
\label{sec:design_validation}

We empirically justify the specific architectural instantiations of each component.

\begin{figure}[tb]
  \centering
  \begin{minipage}[b]{0.66\linewidth}
    \centering
    {\small\itshape Layer-wise}\\[4pt]
    \begin{subfigure}[b]{0.47\linewidth}
      \centering
      \includegraphics[width=\linewidth]{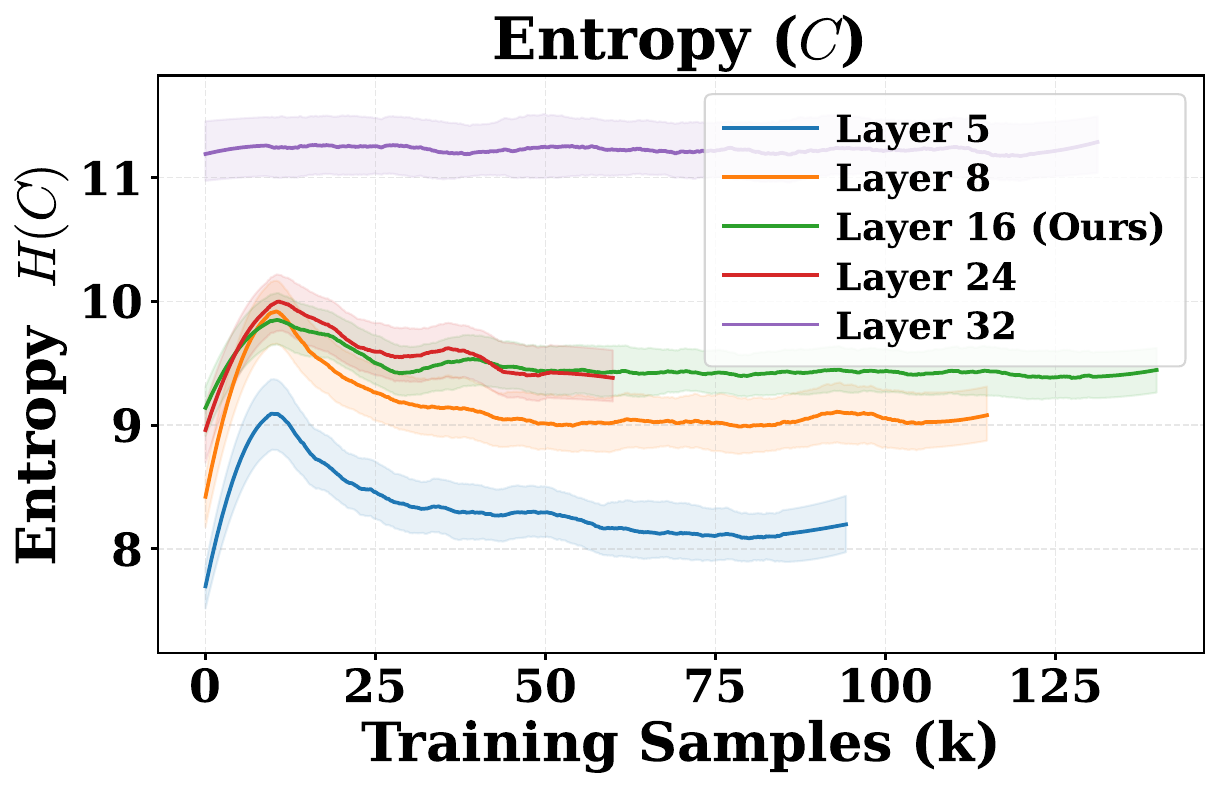}
      \caption{Entropy}
      \label{fig:layer_entropy}
    \end{subfigure}
    \hfill
    \begin{subfigure}[b]{0.49\linewidth}
      \centering
      \includegraphics[width=\linewidth]{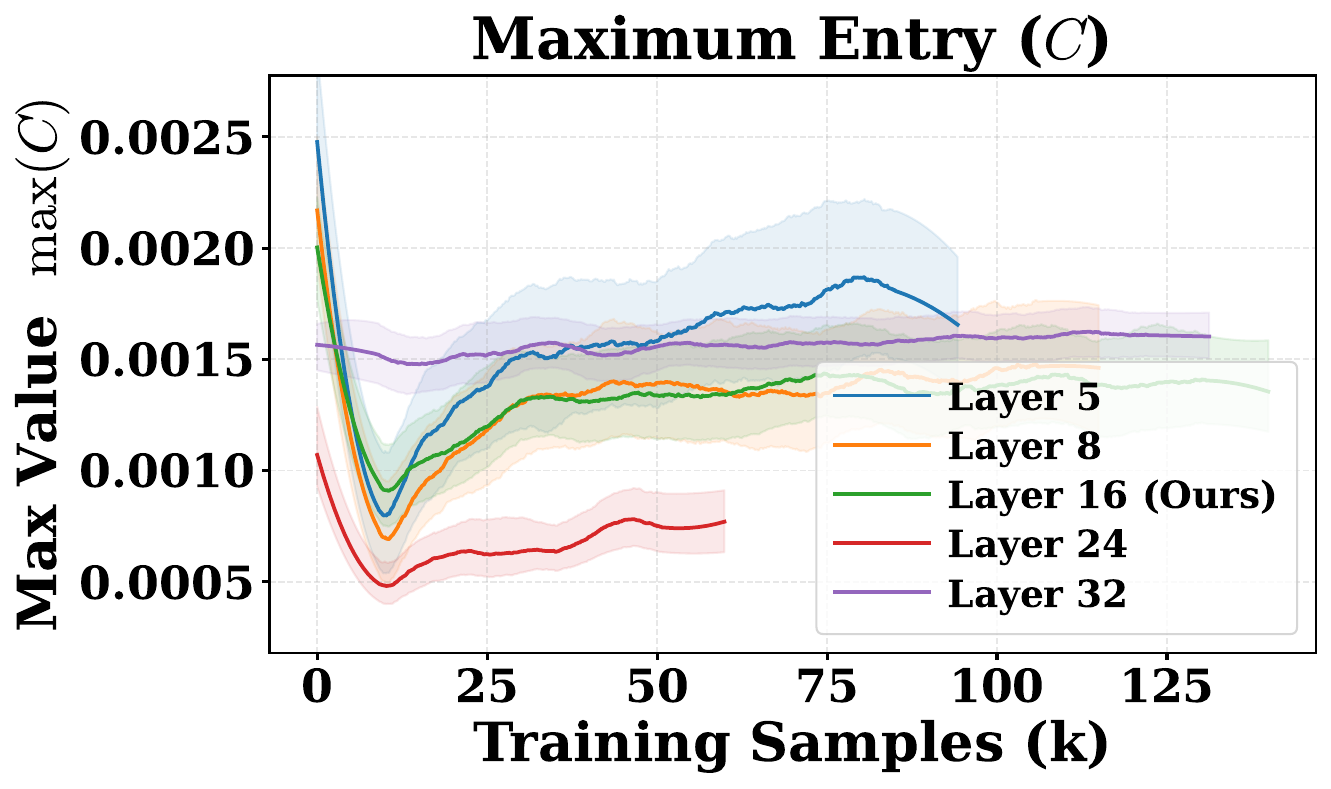}
      \caption{Max attention}
      \label{fig:layer_max}
    \end{subfigure}
  \end{minipage}
  \hfill
  \begin{minipage}[b]{0.32\linewidth}
    \centering
    {\small\itshape Head-wise}\\[4pt]
    \begin{subfigure}[b]{\linewidth}
      \centering
      \includegraphics[width=\linewidth]{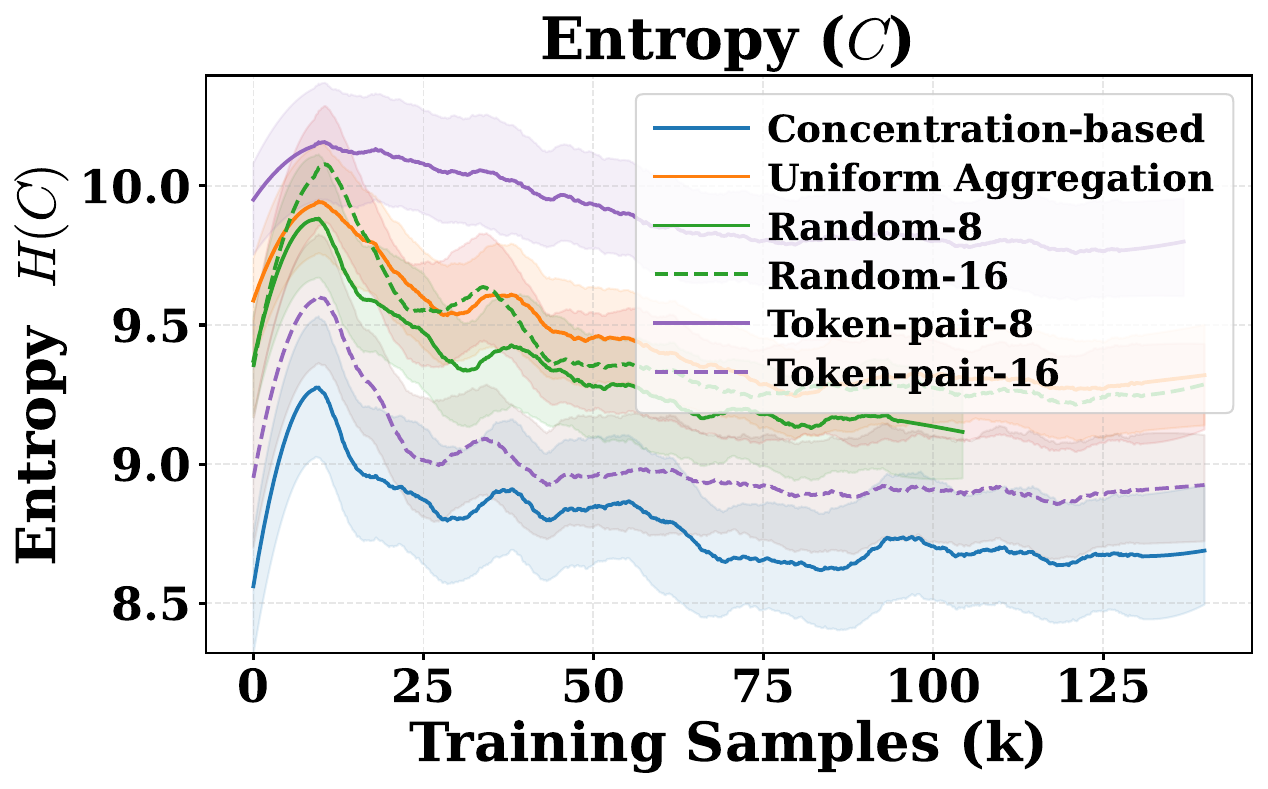}
      \caption{Entropy}
      \label{fig:head_entropy}
    \end{subfigure}
  \end{minipage}
  \caption{Cross-attention analysis from two perspectives.
    \textbf{(a,b) Layer-wise:} intermediate layers achieve balanced
    information richness and peak concentration.
    \textbf{(c) Head-wise (Layer~16):} low-entropy heads yield
    clearer cross-modal alignment.}
  \label{fig:attention_analysis}
\end{figure}
\vspace{-10pt}

\paragraph{Coupling layer selection.}
\cref{fig:attention_analysis}(a,b) presents the layer-wise
cross-attention statistics. Intermediate layers (\eg, Layer~16)
achieve an effective balance between information diversity
(moderate entropy) and signal concentration (high peak attention
values), whereas shallow layers exhibit limited diversity and
deep layers tend to over-abstract, losing discriminative power.

\begin{wraptable}{r}{0.34\linewidth}
  \centering
  \small
  \vspace{-20pt}
  \caption{Sensitivity to the SI-GW weight $\lambda$ on GQA.}
  \label{tab:lambda_sensitivity}
  \setlength{\tabcolsep}{6pt}
  \begin{tabular}{@{}lc@{}}
    \toprule
    $\lambda$ & GQA \\
    \midrule
    $0$ (no SI-GW)              & 60.8 \\
    $10^{-4}$                   & 60.8 \\
    $10^{-3}$                   & 61.5 \\
    $\mathbf{2{\times}10^{-3}}$ & \textbf{61.9} \\
    $5{\times}10^{-3}$          & 61.8 \\
    $10^{-2}$                   & 59.7 \\
    $10^{-1}$                   & 49.3 \\
    \bottomrule
  \end{tabular}
  \vspace{-20pt}
\end{wraptable}

\vspace{-10pt}
\paragraph{Head selection strategy.}
\cref{fig:attention_analysis}(c) compares different head
aggregation strategies. Selecting the top-$k$ lowest-entropy
heads yields the sharpest cross-modal correspondences, while
uniform or random aggregation dilutes the concentrated patterns
that are critical for reliable structural guidance.
\vspace{-10pt}

\paragraph{Token filtering effectiveness.}
The position-aware filtering mechanism effectively down-weights
non-semantic tokens, including function words and early-position
context tokens (\eg, system prompts), thereby isolating the
semantically meaningful tokens required for constructing an
accurate relational geometry.

\paragraph{SI-GW weight $\lambda$.}
\cref{tab:lambda_sensitivity} reports the effect of the regularization
strength $\lambda$. The method is robust across
$\lambda\in[10^{-3},5{\times}10^{-3}]$ and performs best at
$\lambda{=}2{\times}10^{-3}$; as $\lambda$ increases further, the SI-GW
term progressively overrides the language-modeling objective and accuracy
declines.

\vspace{-10pt}
\section{Conclusion}
\label{sec:conclusion}
\vspace{-5pt}
We have identified a critical gap in multimodal alignment: existing methods align individual concepts but fail to preserve relational structures essential for compositional reasoning. To address this, we propose MIRROR, which applies a lightweight variant of Gromov--Wasserstein distance as a geometric regularizer to explicitly align distance structures between visual and textual representations. Through principled hierarchical design strategies, MIRROR induces systematic geometric restructuring of visual space, transforming it to mirror language's relational geometry. Extensive experiments demonstrate that this approach enhances relational reasoning across multiple task levels while fully preserving general vision-language capabilities without additional inference cost. These results establish geometric structure alignment as a promising direction for robust multimodal understanding.


\section*{Acknowledgements}
The authors would like to thank the anonymous reviewers for their valuable comments and suggestions. This work was partially supported by the National Key Research and Development Program of
China (No. 2021YFA1000900), and the National Natural Science Foundation of China (No. 62272432,
No. 62432016).

%
%
\bibliographystyle{splncs04}
\bibliography{main}
\end{document}